\documentclass[letterpaper, 10 pt, conference]{ieeeconf}  

\IEEEoverridecommandlockouts                              

\usepackage{graphics} 
\graphicspath{{gfx/}}
\usepackage{epsfig} 
\usepackage{mathptmx} 
\usepackage{times} 
\usepackage{amsmath} 
\usepackage{amssymb}
\usepackage{soul,color}
\usepackage{url} 
\usepackage{siunitx}
\usepackage{subcaption}
\usepackage{scalerel}

\usepackage{tikz}
\usetikzlibrary{shapes}

\newtheorem{theorem}{Theorem}

\usepackage{xspace}
\usepackage{soul}
\usepackage{mathtools}
\usepackage{makecell}
\usepackage{todonotes}
\usepackage{nicefrac}

\newcommand{\ra}{\rightarrow}

\newcommand{\la}{\leftarrow}

\newcommand{\ie}{\unskip, i.\,e.,\xspace}
\newcommand{\eg}{\unskip, e.\,g.,\xspace}

\newcommand{\sut}{\text{s.\,t.\,}}

\newcommand{\nrm}[1]{\left\lVert#1\right\rVert}

\newcommand{\abs}[1]{\left\lvert#1\right\rvert}

\newcommand{\N}{\ensuremath{\mathbb{N}}}

\newcommand{\R}{\ensuremath{\mathbb{R}}}

\newcommand{\X}{\ensuremath{\mathbb{X}}}

\newcommand{\U}{\ensuremath{\mathbb{U}}}

\newcommand*\diff{\mathop{}\!\mathrm{d}}

\newcommand{\spc}{\ensuremath{\,\,}}

\DeclareMathOperator*{\argmin}{arg\,min}

\definecolor{dgreen}{rgb}{0.0, 0.5, 0.0}

\DeclareMathOperator*{\concat}{\scalerel*{+}{\textstyle\sum}}

\newtheorem{dfn}{Definition}
\newtheorem{asm}{Assumption}

\newtheorem{lem}{Lemma}

\newtheorem{crl}{Corollary}
\newtheorem{rem}{Remark}

\usepackage{multirow}

\usepackage{amsmath}

\usepackage{algorithm}
\usepackage{algorithmic}

\usepackage{svg}

\title{\LARGE \bf
A generalized stacked reinforcement learning method for sampled systems
}

\author{Pavel Osinenko, Dmitrii Dobriborsci, Grigory Yaremenko, Georgiy Malaniya}

\begin{document}

\maketitle
\thispagestyle{empty}
\pagestyle{empty}

\begin{abstract}
A common setting of reinforcement learning (RL) is a Markov decision process (MDP) in which the environment is a stochastic discrete-time dynamical system.
Whereas MDPs are suitable in such applications as video-games or puzzles, physical systems are time-continuous.
A general variant of RL is of digital format, where updates of the value (or cost) and policy are performed at discrete moments in time.
The agent-environment loop then amounts to a sampled system, whereby sample-and-hold is a specific case.
In this paper, we propose and benchmark two RL methods suitable for sampled systems.
Specifically, we hybridize model-predictive control (MPC) with critics learning the optimal Q- and value (or cost-to-go) function.
Optimality is analyzed and performance comparison is done in an experimental case study with a mobile robot.
\end{abstract}

\section{Introduction}
\label{sec:intro}

RL showed remarkable performance in puzzles and video games \cite{silver18, berner2019dota, Vinyals2019}, while also having demonstrated certain success in industrial tasks \cite{kober2013reinforcement, luong2019applications}.
Still, plain RL settings face performance guarantee requirements in industry, traditionally being addressed via classical control, such as MPC \cite{Lazic18Advances, Hrovat12development, Forbes15Model, kouro2015model, vazquez2016model}.
There is {some} promising research on hybridization of MPC with RL and, generally, learning-based methods.
{Offline learning hybrids of RL and MPC have been proposed in an attempt to augment the learning based estimation peculiar to RL with explicit computations of MPC (see \eg \cite{Berkenkamp2017safe, Song20learning, napat20practical, hoeller20deep}). Certain challenging control engineering tasks have been handled using novel online RL-MPC approaches \cite{lenz15deep}, \cite{Drews17aggressive}. An online RL-MPC hybrid is known to produce safe policies for general affine systems with a quadratic stage cost \cite{Zanon2021safe}. Another online RL-MPC approach has been studied for nonlinear systems with discounted stage cost \cite{bhardwaj20blending}.}
It can be seen that, although RL and MPC start off with fundamentally different {optimal control problems (finite horizon vs. infinite horizon} \cite{primbs1999nonlinear}), there is plenty of options to blend them for the sake of gaining the advantages of both.
In this work, we study so-called stacked RL first proposed in \cite{Osinenko2017stacked} in the Q-learning format for discrete-time systems. {Stacked RL is an approach that combines model-based computation of trajectories peculiar to MPC with learnable critics peculiar to RL. Unlike conventional methods of RL such as Q-learning and policy iteration, stacked RL is an online-learning method and thus it is intended to work ``out of the box'' without any pretraining, much like in the case of MPC.}
We generalize this approach to sample-and-hold systems and suggest a further variant of it called stacked RL-QV, where the optimal cost-to-go (or value, depending on the problem statement) function approximation is used as a terminal cost.
This work concentrates on studying the performance of the suggested agents in experiments with a mobile robot, while the question of incorporating safety and stability constraints was addressed in \cite{osinenko2020reinforcement, beckenbach2020closed, beckenbach2018constrained}.

The rest of the paper is organized as follows. 
A theoretical background is addressed in Section \ref{sec:preliminaries}.
The description of the new stacked RL-QV algorithm and its theoretical analysis are given in Section \ref{sec:groundwork}.
Section \ref{sec:transition-to-practice} is concerned with a particular practical realization of the algorithm.
Experimental results and comparison with the stacked Q-learning as per \cite{Osinenko2017stacked}, in the sample-and-hold variant, are given Section \ref{sec:case-study}.

\textbf{Notation.}
For any $z$, denote $\{z_{i|k}\}_{i=1}^N \triangleq \{z_{1|k}, \dots, z_{N|k}\} \triangleq \{z_k, \dots, z_{k+N-1}\}$, $N \in \N$, if a starting index $k \in \N$ is emphasized; otherwise, it is just $\{z_i\}^N = \{z_1, \dots, z_N\}$. 
If $N$ in the above is omitted, the sequence is considered infinite.
The notation $[N]$ will mean the set $\{1, 2, \dots, N\}$.

\section{Theoretical foundation}
\label{sec:preliminaries}

We consider an environment as a general non-linear dynamical system
\begin{equation}
	\label{eqn:sys}
	\tag{sys}
    \dot x = f(x, u),
\end{equation}
where $x \in \R^n$ is the state, $u \in \R^m$ is the control action $u$, and $f: \R^n \times \R^m \ra \R^n$ is the system dynamics function.
The goal is to solve the following infinite-horizon optimal control problem:
\begin{equation}
    \begin{aligned}
        \min_{\rho \in \mathcal U} J^{\rho} \left(x_0 \right) := \int \limits_{0}^{\infty} {e^{-\gamma t}}r \left(x(t), \rho(x(t)) \right) \diff t, \ {\gamma \in [0, +\infty)},
    \end{aligned}
\end{equation}
where $x_0$ is the initial state, $r$ is the running cost, $\rho: \R^n \ra \R^m$ is a control policy {that belongs to some class of admissible policies $\mathcal U$ (we assume that all policies in $\mathcal U$ imply the convergence of the above integral and there exists a minimizer). Note, that unlike \cite{bhardwaj20blending} our setup allows for undiscounted stage costs.}
This problem is given in the cost minimization format, as it is common in control, while RL usually treats $r$ as a reward and aims at its maximization.
The integral cost is also called \textit{cost-to-go} (analogously, reward-to-go).
The optimum of the cost-to-go satisfies the Bellman optimality principle:
\begin{equation}
\begin{aligned}
	& \forall t \in [0, \infty) \ :\\ 
	& J^*(x) := \min_{\rho \in \mathcal U} \left\{ \int \limits_{0}^{t} {e^{-\gamma t}}r(x^\rho(\tau), \rho(x^\rho(\tau)) \diff \tau + J^*(x^\rho(t)) \right\},
\end{aligned}
\end{equation}
where $x^{\rho}(t)$ is the trajectory under the policy $\rho$.
The respective minimizer $\rho^*$, that corresponds to $J^*$, is {a} (globally) optimal policy.
Let $\eta$ be an arbitrary policy on $[0,t]$ ({head policy}).
Then, the Bellman optimality principle can be rewritten in terms of a Q-function with an $\eta$ head:
\begin{equation}
    \begin{aligned}
        Q(x| \eta) & := \int \limits_0^t {e^{-\gamma t}}r(x^\eta(\tau), \eta(\tau)){\diff t} + J^*(x^\eta(t)) \\
	    J^*(x) & = \min \limits_\eta Q(x| \eta). 
    \end{aligned}
\end{equation}
Notice that there is, in general, no limiting case of a Q-function in the time-continuous case \ie where the head policy boils down to a single action.




\section{Groundwork of the approach}
\label{sec:groundwork}

In this section, we discuss stacked RL in a sampled setting.
Let $\delta$ be a sampling step size so that the policy is updated every $\delta$ units of time.
Fix an index $k \in \N$ and let us count sampled policies $\rho_{i|k}$ on $\delta$-intervals via an index $i \in \N$ starting at $k$.
Then, the Q-function at the $i-1$st step reads:
\begin{equation}
    \begin{aligned}
        & Q^{\delta}(x((k+i-1)\delta)| \rho_{i|k}) =\\
        & \int \limits_{(k+i-1)\delta}^{(k+i)\delta} {e^{-\gamma t}}r(x^{\rho_{i|k}}(t), \rho_{i|k})\diff t + J^*(x((k+i)\delta)),
    \end{aligned}
\end{equation}
where $x((k+i-1)\delta)$ is the starting state inside the sampling interval and $Q^{\delta}$ is the Q-function there.
Let us use the shorthand notation, for a $j \in \N$, $x_{j} := x(j \delta)$.
Notice that a sampled state trajectory satisfies, for $t \in [(k+i-1)\delta, (k+i)\delta]$, 
\begin{align}
	x^{\rho_{i|k}}(t) := & x_{i|k} + \int \limits_{(k+i-1)\delta}^{(k+i)\delta} f(x^{\rho_{i|k}}(\tau), \rho_{i|k}) \diff \tau.
\end{align}
In the stacked RL setting, we consider a finite horizon of the described $\delta$-intervals, starting from every step $k$.
Let the horizon be some $N \in \N$.
Then, let $i$ run over $[N-1]$.
Let us consider a stack of Q-functions as follows:
\begin{equation}
    \begin{aligned}
        & \bar Q(x_k|\{\rho_{i|k}\}_{i=1}^{N-1}) := \sum \limits_{i=1}^{N-1} Q^{\delta} (x_{i|k}, \rho_{i|k}). \\
    \end{aligned}
\end{equation}
The optimization problem of the \textit{stacked RL-QV} is suggested in the following form, while adding a $J^*$-terminal cost (cf. stacked Q-learning \cite{Osinenko2017stacked}):
\begin{equation}
    \begin{aligned}
        \min \limits_{\{ \rho_{i|k} \}_{i=1}^{N-1}} \bigg\{ & \sum \limits_{i=1}^{N-1} Q^{\delta}(x_{i|k})| \rho_{i|k}) + J^*(x_{k+N-1}) \bigg\}.
    \end{aligned}
\end{equation}
Thus, the minimization is over a finite stack of sampled policies $\{\rho_{i|k}\}^{N-1}_{i=1}$.  
The next theorem shows that the optimal policy resulting from the Stacked RL-QV yields the globally optimally policy \ie the minimizer of $J^*$.

\begin{rem}
``Collapse of Q-learning" is a phenomenon that occurs when Q-learning is applied to discretized continuous-time systems.
If $\delta$ is small, then $Q^\delta(x, u) \approx J^*(x)$, which makes the updates ill-behaved \cite{Tallec19making}.
Using a stack of Q-functions to perform such an update remedies this issue.
Observe that if for some fixed parameter $a$ one were to choose $N$ in such a way that $N > \frac{a}{\delta}$, then this would prevent the collapse of the stack $\bar{Q}(\cdot)$ to the $N$-fold optimal cost-to-go function $N\cdot J^*(\cdot)$.
{Approaches to remedy the collapse of Q-learning are known.
For instance, in \cite{Tallec19making} and \cite{Doya00reinforcement} the Q-function is replaced with the advantage function that carries the same information but does not collapse as $\delta$ tends to $0$.
Deep advantage updating implies that the optimal value function (in our context, cost-to-go) be too learned alongside with the advantage function $A$.
Although the main purpose of stacking in the current work is not to avoid Q-learning collapse in a continuous time setting \textit{per se}, the idea of stacking could be extended to the case of learning the advantage function.
This is however beyond the scope of this work.
Furthermore, we explicitly stress the sampled character of the agent \ie $\delta$ is always assumed strictly positive.
}
\end{rem}

\begin{theorem}
\label{thm:stacked-QL}

{Let $\{x^*_{k + i}\}_{i = 0}^{N - 1}$ be an optimal trajectory starting at $x_k$ and for any other trajectory $\{\bar{x}_{k + i}\}_{i = 0}^{N - 1}$, such that $\bar{x}_k = x_k$,  it holds that $\sum_{i = 1}^{N - 1} J^*(x^*_{k + i}) \leq \sum_{i = 1}^{N - 1} J^*(\bar{x}_{k + i})$. Then}

\begin{equation}
	\label{lemma:V*-leq-Gamma-1}
\begin{aligned}
        & \min \limits_{\{ \rho_{i|k} \}_{i=1}^{N-1}} \left(\sum \limits_{i=1}^{N-1} Q^{\delta}(x_{i|k}| \rho_{i|k}) + J^*(x_{k+N-1})\right) = \\
        & \sum_{i=1}^{N-1} \min \limits_{\rho_{i|k}} Q^{\delta}(x_{i|k}| \rho_{i|k}) + J^*(x_{k+N-1})\\
    \end{aligned}
\end{equation}
\end{theorem}
{
\begin{crl}
Minimizing $\bar{Q}$ yields an optimal policy under the aforementioned assumption.
\end{crl}
}
{
\begin{rem}
Note, that
\begin{equation}
\begin{aligned}
       & \bar Q(x_k|\{\rho_{i|k}\}_{i=1}^{N-1}) =  \\
       & \underbrace{\int_{k\delta}^{(k + N - 1)\delta}e^{-\gamma t}r\big(x(t), \rho(x(t))\big)\diff t}\limits_{\text{MPC cost function}} + \sum \limits_{i=1}^{N - 1} J^*(x_{k+i}). 
    \end{aligned}
\end{equation}
One can observe that $\bar{Q}$ is essentially MPC cost complemented with additional information from beyond the prediction horizon. We have the learning aspects of RL combined with the act of looking several steps ahead, peculiar to MPC. This is the main motivation of stacked RL.
\end{rem}}

\begin{asm}
\label{asm:main}
Let $f(x, u)$ from \eqref{eqn:sys} be locally Lipschitz continuous in $x$ uniformly in $u \in \U$, let $\U$ be compact and let $\abs{r(\cdot, \cdot)}$, $\nrm{f(\cdot, \cdot)}$  be upper semi-continuous.
\end{asm}

The following theorem illustrates how for a sufficiently high sampling frquency the difference between the two metrics in \eqref{lemma:V*-leq-Gamma-1} is negligible.

{\begin{theorem}
Assumption \ref{asm:main} implies
\begin{equation}
	\label{thm:delta}
\begin{aligned}
        & \min \limits_{\{ \rho_{i|k} \}_{i=1}^{N-1}} \left(\sum \limits_{i=1}^{N-1} Q^{\delta}(x_{i|k}| \rho_{i|k}) + J^*(x_{k+N-1})\right) - \\
        & \left(\sum_{i=1}^{N-1} \min \limits_{\rho_{i|k}} Q^{\delta}(x_{i|k}| \rho_{i|k}) + J^*(x_{k+N-1})\right) \xrightarrow{\delta \rightarrow 0} 0. \\
    \end{aligned}
\end{equation}
\end{theorem}}

In the next section, we discuss a particular realization of the stacked RL-QV algorithm using a double-critic for a sample-and-hold system, which is the simplest form of sampled system, in which the sampled policies simplify just to constant actions.

\section{A realization in sample-and-hold setting}
\label{sec:transition-to-practice}

\begin{figure*}[ht!]
\centering
\begin{subfigure}[b]{0.49\textwidth}
    \includegraphics[width=\textwidth]{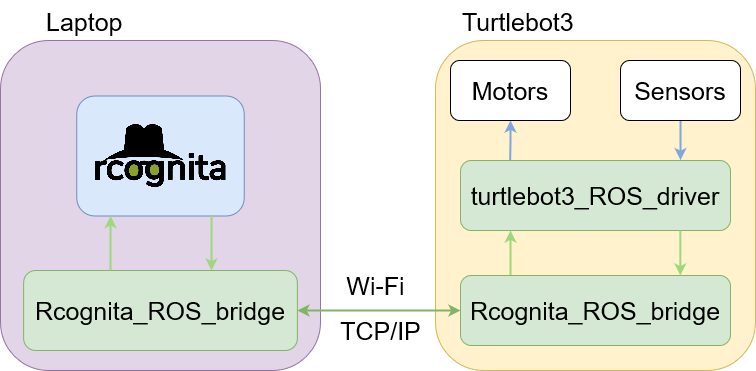}
    \caption{Core of software setup using \texttt{rcognita} Python framework.}
\end{subfigure}
\begin{subfigure}[b]{0.49\textwidth}
    \includegraphics[width=\textwidth]{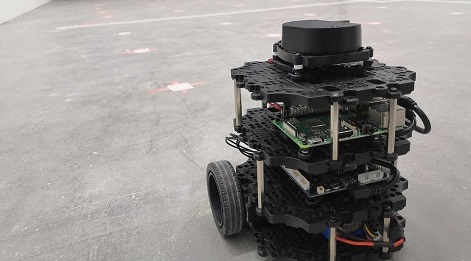}
    \caption{Robotis TurtleBot3 and experimental environment}
\end{subfigure}
\caption{Software and hardware experimental setup}
\label{fig:softhard}
\end{figure*}




The sample-and-hold setting can be described as follows:
\begin{equation}
	\label{eqn:full-sys-sh}
	\tag{sys-$\delta$}
    \begin{aligned}
        & \mathcal D x = f(x, u^\delta), \spc x(0) = x_0 \\
        & x_{i|k} := x((k+i-1) \delta), \spc k \in \N \\
        & u^\delta(t) \equiv u_{i|k} = \rho(x_{i|k}), \spc t \in \left[ k\delta, (k + i)\delta \right],
    \end{aligned}
\end{equation}
where $\mathcal D$ is a suitable derivative operator, $u^\delta$ is a control signal sampled at $\delta$-steps, $\rho: \R^n \ra \R^m$ is a policy to be applied sample-wise, and $\delta > 0$ is the sampling time as before.
For any $k \in \N$, the state $x^{u_k}(t)$ at $t \geq k \delta$ under $u_k$ satisfies
\begin{equation}
    \label{eqn:sh}
     x^{u_k}(t) := x_k + \int \limits_{k \delta}^t f \left(x(\tau), u_k \right) d\tau.
\end{equation}
When predicting these sampled states in practice, one has to use an approximation, in general \eg a numerical integration scheme.
The Euler scheme is the simplest one:
\begin{equation}
    \label{eqn:Euler}
    \hat x_{i+1|k} = \hat x_{i|k} + \delta f(\hat x_{i|k}, u_{i|k}),
\end{equation}
where $\hat x_{i|k}$ denotes the approximate state.
Whereas tabular methods of RL can effectively compute the Q- and/or value function, they are not suitable for online application, which is usually tackled using temporal-difference (TD).
An actor-critic structure with neural networks is then employed to approximate the optimal Q- and/or value (or cost-to-go) function using the TD.
In this work, we consider a double-critic structure, that approximates both, as follows:
\begin{equation}
    \begin{aligned}
        \hat{Q}(x,u; w) & := \langle w, \varphi_Q (x, u) \rangle, \hat J(x; v) := \langle v, \varphi_V (x) \rangle \\
        \varphi_Q (x, u) & := \concat_{k=0}^{n + m - 1}\concat_{i - j = k}\left((x^T, u^T)^T(x^T, u^T)\right)_{ij}, \\
        \varphi_V (x) & := \concat_{k=0}^{n - 1}\concat_{i - j = k}\left(xx^T\right)_{ij}, \\
    \end{aligned}
\end{equation}
where $w, v$ are the weights of the $\hat Q$- and $\hat J$-critic, respectively, $\varphi_Q, \varphi_V$ are the corresponding activation functions, $\concat$ is the operator of concatenating vectors.
We consider here shallow networks only for the ease of notation -- deep topologies can be used analogously.

The critic objective can be formulated via an array of TDs formed from an experience replay as follows:
\begin{equation}
    \begin{aligned}
	\label{eqn:critic-cost}
	& \min \limits_{w_{k}, v_k} J^c(w_{k}, v_k), \\
	\textrm{where} \\
	& J^c(w_{k}, v_k) := \dfrac{1}{2} \sum \limits_{j=1}^{M-1} e_j^2(w_k) + \dfrac{1}{2} \sum \limits_{j=1}^{M-1}(e^2_j(v_k)), \\
	& e_j(w_{k}) = w_{k} \varphi( x_{j|k-M}, u_{j|k-M} ) \\
	& - {e^{-\gamma\delta}} w_{k-1} \varphi(x_{j+1|k-M}, u_{j+1|k-M}) - r(x_{j|k-M}, u_{j|k-M}), \\
	& e_j(v_{k}) = v_{k} \varphi( x_{j|k-M}, u_{j|k-M} ) \\
	& - {e^{-\gamma\delta}} v_{k-1} \varphi(x_{j+1|k-M}, u_{j+1|k-M}) - r(x_{j|k-M}, u_{j|k-M}), \\
	\end{aligned}
\end{equation}
and $w_{k}, v_{k}$ are the vectors of the critic neural network weights to be optimized over, $w_{k-1}, v_{k-1}$ are the vectors of the weights from the previous time step, $M$ is the size of the experience replay. 

The actor objective can be expressed in the following form, according to the stacked RL-QV principle:
\begin{equation}
    \begin{aligned}
        \min_{\{u_{i|k}\}_{i=1}^{N-1}} \spc & J^a \left( x_k|\{u_{i|k}\}_{i=1}^{N-1}; w_{k-1} \right) \\
        & = \sum_{i=1}^{N-1} \hat{Q}(\hat x_{i|k}, u_{i|k}; w_{k-1}) + \hat J(\hat x_{N|k}; v_{k-1}), \\
        \sut \spc & \hat x_{i+1|k} = \hat x_{i|k} + \delta f(\hat x_{i|k}, u_{i|k})
    \end{aligned}
\end{equation}   

The overall actor-critic realization of the stacked RL-QV in a sample-and-hold setting is summarized in Algorithm \ref{alg:setup}.
{
\begin{rem}
Note, that despite the fact that the critics are updated via ordinary experience replay, the online planning phase (see line 4 in Algorithm \ref{alg:setup}) involves explicitly computing future states, using the known model of the system. Ordinary MPC would evaluate a policy $\{\rho_{i|k}\}_{i=1}^{N-1}$ by first computing the corresponding trajectory $\{x_{k + i - 1}\}_{i=1}^N$ and then using the value of $\sum_{i=1}^{N - 1}r(x_{k + i - 1}, \rho_{i|k})$ to tell how good the policy is. Likewise,  RL-QV measures the goodness of a policy $\{\rho_{i|k}\}_{i=1}^{N-1}$ by first computing the corresponding trajectory $\{x_{k + i - 1}\}_{i=1}^N$ and then evaluating the obtained trajectory via critics (instead of computing the partial cost, like MPC would). This essentially makes RL-QV a version of MPC with stage costs substituted for critics.
\end{rem}
}

\begin{algorithm*}
    \caption{Actor-critic, sample-and-hold realization of stacked RL-QV}
    \label{alg:setup}
    \begin{algorithmic}[1]
    \STATE {\bfseries Init:} $\hat Q$, $\hat J$, $\delta$, $N$, $x_k \la x_{\textrm{init}}$ $w^*_{k-1} \la w_{\textrm{init}}$, $v^*_{k-1} \la v_{\textrm{init}}$
    \STATE $k=0$
    \WHILE{$true$}
        \STATE Update actor: $\{u^*_{i|k}\}_{i=1}^{N-1} \la \argmin \limits_{{\{u_{i|k}\}}^{N-1}_{i = 1}} \left( \sum \limits_{i=1}^{N-1} \hat Q (\hat x_{i|k}, u_{i|k}, w^*_{k-1}) + \hat J(\hat x_{N|k}; v^*_{k-1}) \right)$
        \STATE Apply $u^*_{1|k}$ to the system
        \STATE Update critics: $w^*_k, v^*_k \la \argmin \limits_{w_k, v_k} \left(\dfrac{1}{2} \sum \limits_{j=1}^{M-1} e_j^2(w_k) + \dfrac{1}{2} \sum \limits_{j=1}^{M-1}(e^2_j(v_k))\right)$ \\
        \STATE $k \la k+1$
    \ENDWHILE
    \end{algorithmic}
\end{algorithm*}


\begin{rem}
	\label{rem:complemented_stack}
	To lift the assumption $\sum_{i = 1}^{N - 1} J^*(x^*_{k + i}) \leq \sum_{i = 1}^{N - 1} J^*(\bar{x}_{k + i})$ of Theorem \ref{thm:stacked-QL}, the Q-function stack may be complemented as follows:
	\begin{multline}
	\label{eqn:complemented}
	Q^\circ(x_k|\{\rho_{i|k}\}_{i=1}^{N-1}) := \\
	 \sum \limits_{i=1}^{N-1} Q^{\delta} (x_{i|k}, \rho_{i|k}) + J^*(x_{k+N-1}) + \sum_{i=1}^{N - 1}(N - i)r(x_{k + i - 1}, \rho_{i | k}). \\
	\end{multline}
	In this case, one could \eg take a model $\hat{Q}(x_{i|k},u_{i|k}; w) + (N-i)r(x_{i|k},\rho_{i|k})$.
\end{rem}

\begin{theorem}
Let
\begin{multline}
 \{ \bar \rho_{i|k} \}_{i=1}^{N-1} := \argmin \limits_{\{ \rho_{i|k} \}_{i=1}^{N-1}} Q^\circ(x_k|\{\rho_{i|k}\}_{i=1}^{N-1}),\\
\end{multline}
then $\{\bar \rho_{i|k} \}_{i=1}^{N-1}$ are optimal control inputs.
\end{theorem}

\section{Experimental study}
\label{sec:case-study}

\begin{figure}
    \centering
    \includegraphics[width=0.45\textwidth]{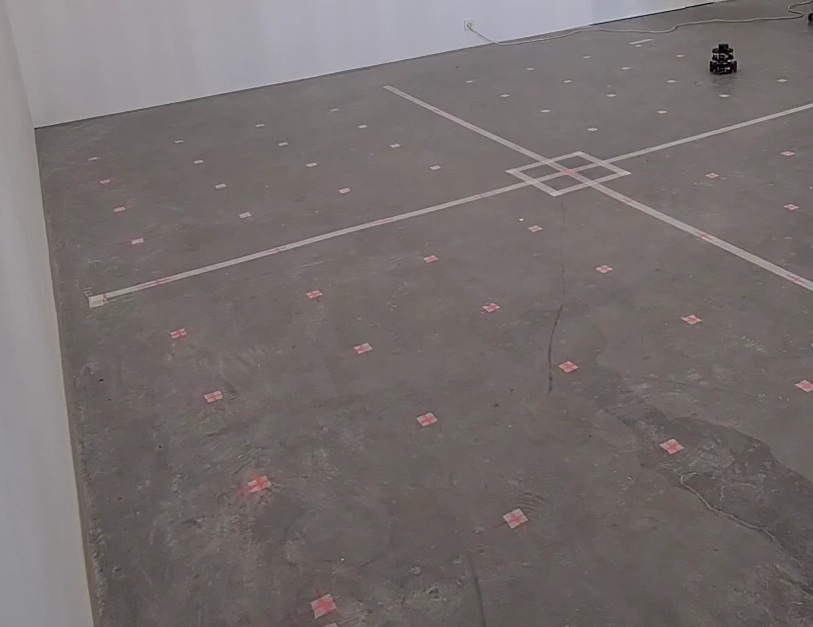}
    \caption{The experimental polygon. The target is in the center}
    \label{fig:cg}
\end{figure}


%
%

The stacked RL-QV was tested in mobile robot parking problem and compared to the stacked Q-learning based on \cite{Osinenko2017stacked} (further called stacked RL-Q for consistency) and an MPC agent.

The experiment consisted of a total of 150 runs.
The robot started from different positions on a grid over an experimental polygon with a side of 4 m, turned away from the center -- the target (see Fig \ref{fig:cg}).
The goal was to park the robot at the target while achieving a desired orientation.

The model of the three-wheel robot was assumed in the following form:
\begin{equation}
\label{eqn:3w-robot}
\begin{cases}
\dot x = v\cos\alpha \\
\dot y = v \sin \alpha \\
\dot{\alpha} = \omega,
\end{cases}
\end{equation}
where $x$ is the $x$-coordinate in [m], $y$ is the $y$-coordinate in [m], $\alpha$ is the orientation angle in [rad], $v$ is the linear velocity in [m/s], $\omega$ is the angular velocity in [rad/s].
The action is $u = [v | \omega]$.


The stage cost was taken in the following quadratic form:
\begin{equation}
    \begin{aligned}
        & r = \chi^\top R \chi, \\
    \end{aligned}
\end{equation}
where $\chi$ is equal either $\chi_Q$ or $\chi_J$ -- the Q- and cost-to-go function regressor, respectively -- $\chi_Q := [(x, y, \alpha)| u]$, $\chi_J := (x, y, \alpha)$, $R$ is a diagonal positive-definite matrix.

All the experiments were performed on a Robotis TurtleBot3 using a Robot Operating System (ROS) and a Python framework for RL called \texttt{rcognita} (see Fig. \ref{fig:softhard} and the Github link \texttt{github.com/pavel-osinenko/rcognita} for more details).
Experiments were carried out in the following setups:
\begin{enumerate}
	\item Long-sighted horizon: $N = 12, \delta = 0.1 s$. Here, the effective prediction horizon equals 1.2 seconds (see fig.~\ref{fig:long});
	\item Short-sighted horizon: $N = 6, \delta = 0.1 s$. The effective prediction horizon equals 0.6 seconds (see fig.~\ref{fig:short}).
\end{enumerate}

For each start on the grid, the accumulated stage costs $J_{\text{RL-Q}}, J_{\text{RL-QV}}, J_{\text{MPC}} $ were computed.
The accumulated cost relationships
\begin{equation}
	J_{\text{RL-Q/MPC}} = \frac{ J_{\text{RL-Q}} } { J_{\text{MPC}} } \cdot 100\%,
	\label{eq:rql}
\end{equation}
\begin{equation}
	J_{\text{RL-QV/MPC}} = \frac{ J_{\text{RL-QV}} } { J_{\text{MPC}} } \cdot 100\%,
	\label{eq:rql-v}
\end{equation}
where then computed as performance marks.

The results of the cost difference percentage \eqref{eq:rql} are shown in Fig.~\ref{fig:rql-long}, \ref{fig:rql-short}, while of \eqref{eq:rql-v} -- in Fig.~\ref{fig:rqlv-long}, \ref{fig:rqlv-short}.

It was observed that both the RL agents generally outperformed MPC in terms of the accumulated stage cost for both horizon setups (see Fig.~\ref{fig:long}--\ref{fig:short}).
At the same time, the QV variant showed better results than the Q one.
Notice that both have the same computational complexity for equal horizon.
This indicated certain merits of the QV variant in the conducted experiment.

{
The results indicate that the advantage of the proposed stacked approaches over MPC strongly depends on the area of the state space, from which that agent begins its learning. It is also notable that both RL-Q and RL-QV seem to outperform MPC consistently for longer prediction horizons. Thus in a practical scenario, when MPC is done with the largest horizon that can be computed in a reasonable ammount of time, RL-Q and RL-QV seem to be superior to MPC.}

{\begin{rem}
{It is worth pointing out that, despite the fact that the considered setting does not account for noise and other kinds of stochasticity, the proposed approach managed to perform well even in the presence of the natural noises that occured in the testing environment. Thus it indicates that the approach can be utilized in a practical scenario without necessarily introducing explicit robusifying measures. }
\end{rem}}

\begin{figure*}
	\centering
	\begin{subfigure}[t]{0.49\textwidth}
		\includegraphics[width=\linewidth]{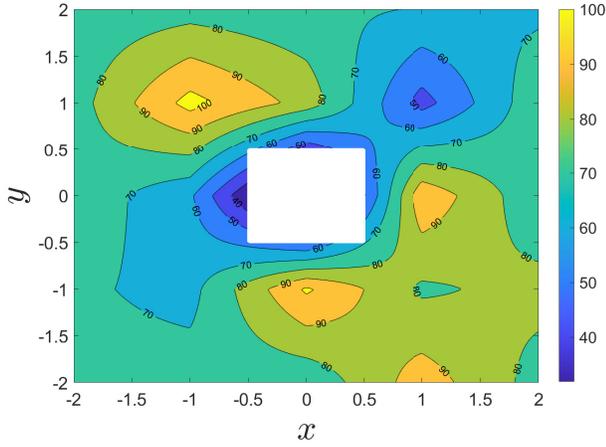}
		\caption{$J_{\text{RL-Q/MPC}}$ for long-sighted horizon. See \eqref{eq:rql}}
		\label{fig:rql-long}
	\end{subfigure}
	\begin{subfigure}[t]{0.49\textwidth}
		\includegraphics[width=\linewidth]{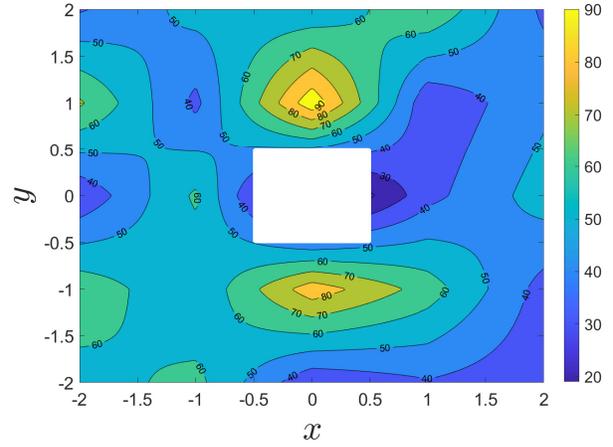}
		\caption{$J_{\text{RL-QV/MPC}}$ for long-sighted horizon. See \eqref{eq:rql-v}}
		\label{fig:rqlv-long}
	\end{subfigure}
	\caption{Contour plot of accumulated cost relationships (interpolated between ticks for better view).}
	\label{fig:long}
\end{figure*}

\begin{figure*}
	\centering
	\begin{subfigure}[t]{0.49\textwidth}
		\includegraphics[width=\linewidth]{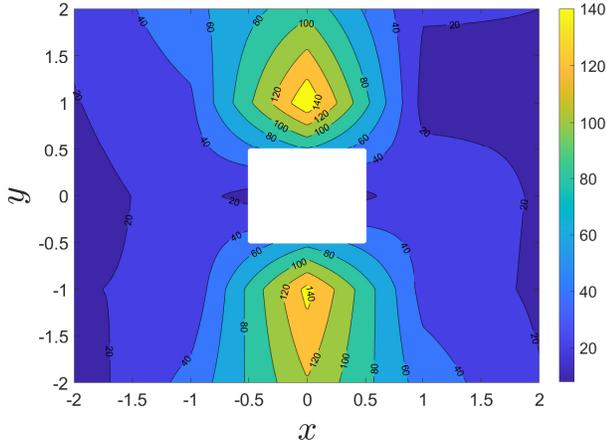}
		\caption{$J_{\text{RL-Q/MPC}}$ for {short}-sighted horizon. See \eqref{eq:rql}}
		\label{fig:rql-short}
	\end{subfigure}
	\begin{subfigure}[t]{0.49\textwidth}
		\includegraphics[width=\linewidth]{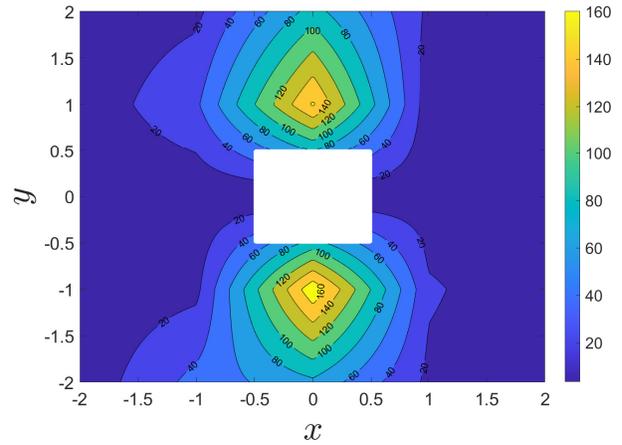}
		\caption{$J_{\text{RL-QV/MPC}}$ for short-sighted horizon. See \eqref{eq:rql-v}}
		\label{fig:rqlv-short}
	\end{subfigure}
	\caption{Contour plot of accumulated cost relationships (interpolated between ticks for better view). \\
{Here $x$ and $y$ denote respectively frontal and lateral offsets from the desired equilibrium at $t = 0$.}}
	\label{fig:short}
\end{figure*}

\appendix
\newtheorem{thma}{Theorem}
\begin{thma}
{Let $\{x^*_{k + i}\}_{i = 0}^{N - 1}$ be an optimal trajectory starting at $x_k$ and for any other trajectory $\{\bar{x}_{k + i}\}_{i = 0}^{N - 1}$, such that $\bar{x}_k = x_k$,  it holds that $\sum_{i = 1}^{N - 1} J^*(x^*_{k + i}) \leq \sum_{i = 1}^{N - 1} J^*(\bar{x}_{k + i})$. Then}

\begin{equation}
	\label{lemma:V*-leq-Gamma}
\begin{aligned}
        & \min \limits_{\{ \rho_{i|k} \}_{i=1}^{N-1}} \left(\sum \limits_{i=1}^{N-1} Q^{\delta}(x_{i|k}| \rho_{i|k}) + J^*(x_{k+N-1})\right) = \\
        & \sum_{i=1}^{N-1} \min \limits_{\rho_{i|k}} Q^{\delta}(x_{i|k}| \rho_{i|k}) + J^*(x_{k+N-1})\\
    \end{aligned}
\end{equation}

\end{thma} 
\begin{proof}
    Let the optimal sampled policy sequence for the Q-function stack be denoted as:
    \begin{equation}
        \begin{aligned}
            & \{ \bar{\rho}_{i|k}^{*}\}_{i=1}^{N-1} := \argmin_{\{ \rho_{i|k}\}_{i=1}^{N-1}} \left(\sum \limits_{i=1}^{N-1} Q^{\delta}(x_{i|k}| \rho_{i|k}) + J^*(x_{k+N-1})\right).
        \end{aligned}
    \end{equation}
    The optimal sampled policy sequence of the element-wise Q-function optimization reads:
    \begin{equation}
        \begin{aligned}
            \{ \rho_{i|k}^{*}\}_{i=1}^{N-1} := & \left\{\argmin \limits_{\rho_{i|k}} Q^{\delta}(x_{i|k}| \rho_{i|k}) \right\}_{i=1}^{N-1}. \\
        \end{aligned}
    \end{equation}
    Denote the corresponding sequences of sampled state trajectories as 
    $\{ \bar{x}_{k + i - 1}\}_{i=1}^{N}$
    and
    $\{ x^*_{k + i - 1}\}_{i=1}^{N}$
    respectively. Note that $\{ x^*_{k + i - 1}\}_{i=1}^{N}$ is in fact an optimal trajectory.
    
Now, observe that    
    
\begin{equation}
\label{eqn:star}
\begin{aligned}
& \sum_{i=1}^{N-1} \min \limits_{\rho_{i|k}} Q^{\delta}(x_{i|k}| \rho_{i|k}) + J^*(x_{k+N-1}) = \\
& \sum_{i=1}^{N-1} Q^{\delta}(x^*_{i|k}| \rho^*_{i|k}) + J^*(x^*_{k+N-1}) = \\ 
& \int_{k\delta}^{(k + N - 1)\delta}e^{-\gamma t}r\big(x^*(t), \rho^*(t)\big)\diff t +  J^*(x^*_{k + N - 1})  + \\ 
& \sum \limits_{i=1}^{N - 1} J^*(x^*_{k+i}).\\
\end{aligned}
\end{equation}
 At the same time
 
\begin{equation}
\label{eqn:bar}
\begin{aligned}
& \min \limits_{\{ \rho_{i|k} \}_{i=1}^{N-1}} \left(\sum \limits_{i=1}^{N-1} Q^{\delta}(x_{i|k}| \rho_{i|k}) + J^*(x_{k+N-1})\right) = \\
& \sum_{i=1}^{N-1} Q^{\delta}(\bar x_{i|k}| \bar \rho_{i|k}) + J^*(\bar x_{k+N-1}) = \\ 
& \int_{k\delta}^{(k + N - 1)\delta}e^{-\gamma t}r\big(\bar x(t), \bar \rho(t)\big)\diff t + J^*(\bar x_{k + N - 1}) +\\ 
& \sum \limits_{i=1}^{N - 1} J^*(\bar x_{k+i}).\\
\end{aligned}
\end{equation}

By optimality we have 
\begin{equation}
\label{eqn:int}
\begin{aligned}
& \int_{k\delta}^{(k + N - 1)\delta}e^{-\gamma t}r\big(x^*(t), \rho^*(t)\big)\diff t +  J^*(x^*_{k + N - 1}) \leq \\
& \int_{k\delta}^{(k + N - 1)\delta}e^{-\gamma t}r\big(\bar x(t), \bar \rho(t)\big)\diff t + J^*(\bar x_{k + N - 1}).
\end{aligned}
\end{equation}

By the theorem's assumption we have
\begin{equation}
\label{eqn:value}
\sum \limits_{i=1}^{N - 1} J^*(x^*_{k+i}) \leq \sum \limits_{i=1}^{N - 1} J^*(\bar x_{k+i}).
\end{equation}

    Now, consider identities \eqref{eqn:star}, \eqref{eqn:bar}, and observe that adding the above inequalities \eqref{eqn:int} and \eqref{eqn:value} yields:
    \begin{equation}
    \label{eqn:less}
\begin{aligned}
& \sum_{i=1}^{N-1} \min \limits_{\rho_{i|k}} Q^{\delta}(x_{i|k}| \rho_{i|k}) + J^*(x_{k+N-1}) = \\
& \int_{k\delta}^{(k + N - 1)\delta}e^{-\gamma t}r\big(x^*(t), \rho^*(t)\big)\diff t +  J^*(x^*_{k + N - 1})  + \\ 
& \sum \limits_{i=1}^{N - 1} J^*(x^*_{k+i}) \leq \\
& \int_{k\delta}^{(k + N - 1)\delta}e^{-\gamma t}r\big(\bar x(t), \bar \rho(t)\big)\diff t + J^*(\bar x_{k + N - 1}) +\\ 
& \sum \limits_{i=1}^{N - 1} J^*(\bar x_{k+i}) =  \min \limits_{\{ \rho_{i|k} \}_{i=1}^{N-1}} \left(\sum \limits_{i=1}^{N-1} Q^{\delta}(x_{i|k}| \rho_{i|k}) + J^*(x_{k+N-1})\right). \\
\end{aligned}
    \end{equation}
    On the other hand, the minimum of the sum is no greater than the sum of successive minima \ie:
    \begin{equation}
    \label{eqn:greater}
        \begin{aligned}
            & \min_{\{ \rho_{i|k}\}_{i=1}^{N-1}} \left(\sum \limits_{i=1}^{N-1} Q^{\delta}(x_{i|k}, \rho_{i|k}) + J^*(x_{k+N-1}) \right) \\
            & \leq \sum_{i=1}^{N-1} \min \limits_{\rho_{i|k}} Q^{\delta}(x_{i|k}, \rho_{i|k})  +  J^*(x_{k+N-1}). \\
        \end{aligned}
    \end{equation}
    Thus, by double inclusion, \eqref{eqn:less} together with \eqref{eqn:greater} imply the required identity
    \begin{equation}
    	\label{eqn:min-stack-sum-min}    
        \begin{aligned}
            & \min \limits_{\{ \rho_{i|k} \}_{i=1}^{N-1}} \left(\sum \limits_{i=1}^{N-1} Q^{\delta}(x_{i|k}, \rho_{i|k}) + J^*(x_{k+N-1})\right) = \\
            & \sum_{i=1}^{N-1} \min \limits_{\rho_{i|k}} Q^{\delta}(x_{i|k}, \rho_{i|k}) + J^*(x_{k+N-1}).
        \end{aligned}
    \end{equation}
\end{proof}
\begin{rem}
The sampled policies $\{\rho^*_{i|k}\}_{i=1}^{N-1}$ and $\{\bar \rho^*_{i|k}\}_{i=1}^{N-1}$  yield the same total cost.
\end{rem}

\begin{rem}
Observe that the assumption $\sum_{i = 1}^{N - 1} J^*(x^*_{k + i}) \leq \sum_{i = 1}^{N - 1} J^*(\bar{x}_{k + i})$ can be replaced with a stronger one of the following kind:
\begin{equation}
J^*(x^*_{k + i}) \leq J^*(\bar{x}_{k + i}), \ i = 0,..., N - 1.
\end{equation}

Although the assumption is stronger, it is evidently easier to verify. Systems that satisfy the latter assumption are not uncommon; consider, for instance a Markov chain described by the following diagram:
\begin{center}
\begin{tikzpicture}[->]
    \node[ellipse,draw] (C) at (10,10) {};
    \node[ellipse,draw] (N) at (8,8.5) {$x_0$};
    \node[ellipse,draw] (Y) at (12,8.5) {};
    \node[ellipse,draw] (K) at (10,7) {};

    \path (N) edge              node[above] {2} (C);
    \path (C) edge              node[above] {2} (Y);
    \path (N) edge              node[below] {10} (K);
    \path (Y) edge [loop right] node {0} (Y);
        \path (K) edge              node[below] {10} (Y);
  \end{tikzpicture}
\end{center}
\end{rem}

\begin{dfn}
$H^\delta(x_0, u)$ is a function that maps $x_0 \in \X$ and $u \in \U$ to the corresponding solution $x(\cdot) : [0, \delta] \rightarrow \X$ of 
\begin{equation}
\left\{ \begin{aligned}
& \dot{x} = f(x(t), u(t)) \\
& x(0) = x_0 \\
& u(t) := u
\end{aligned}\right.
\end{equation}
\end{dfn}

\begin{lem}
\label{lem:bounded-1}
Let Assumption \ref{asm:main} hold. Let $A \subset \X$ and $\U$ be compact sets. Then $H^\delta(A, \U) := \{H^\delta(x_0, u) \ | \ x_0 \in A, \ u \in \U \}$ is compact with respect to the uniform norm $\nrm{\cdot}_u$ in $\text{C}([0, \delta])$, and $\text{Im } H^\delta(A, \U) := \bigcup \{\text{Im } x(\cdot) \ | \ x(\cdot) \in H^\delta(A, \U)\}$ is bounded.
\end{lem}

\begin{proof}
A direct product of compact topological spaces is a compact topological space, thus $A \times \U$ is compact. 

It is known that $H^\delta(\cdot, \cdot)$ is continuous with respect to the uniform norm $\nrm{\cdot}_u$ in $\text{C}([0, \delta])$ (See Theorem 2.6 in \cite{Khalil1996-nonlin-sys}). 

Continuous functions map compact sets to compact sets, thus  $H^\delta(A, \U) = H^\delta(A \times \U)$ is compact. 

Since $\exists c > 0 \ \forall f \in H^\delta(A, \U) \ : \ \nrm{f}_u < c$, then obviously $\exists c > 0 \ \forall f \in H^\delta(A, \U) \ \forall v \in \text{Im }f \ : \ \nrm{v}_2 < c$, which in turn implies that $\text{Im } H^\delta(A, \U)$ is indeed bounded.
\end{proof}

\begin{dfn}
\begin{equation}
H^\delta_n(A, \U) := \underbrace{H^\delta(\text{Im }H^\delta(... \ \text{Im }H^\delta(A, \U) \ ..., \U), \U)}\limits_{n \text{ times}}
\end{equation}
\end{dfn}

\begin{lem}
Let Assumption \ref{asm:main} hold. If $A$ and $U$ are compact, then   
\begin{equation}
\forall n \in \mathbb{N} \ \text{Im }H^\delta_n(A, \U) \text{ is bounded.}
\end{equation}
\end{lem}
\begin{proof}
It is already known that $\text{Im }H^\delta_1(A, \U)$ is bounded. Now, let's assume that $A_n := \text{Im }H^\delta_n(A, \U)$ is bounded. Let $\bar{A}_n$ be the closure of $A_n$. By Lemma \ref{lem:bounded-1}  $\text{Im }H^\delta(\bar{A}_n, \U)$ is bounded. At the same time 
\begin{equation}
A_{n+1} = \text{Im }H^\delta(A_n, \U) \subset \text{Im }H^\delta(\bar{A}_n, \U).
\end{equation}
Thus we have a proof by induction.
\end{proof}

\begin{lem}
Let Assumption \ref{asm:main} hold. Let $\delta \leq \bar \delta$. Then $\text{Im } H^\delta_n(A, \U) \subset \text{Im } H^{\bar \delta}_n(A, \U)$
\end{lem}
\begin{proof}
Let $A^\delta_n := \text{Im } H^\delta_n(A, \U)$ and $A^{\bar \delta}_n := \text{Im } H^{\bar \delta}_n(A, \U)$.
Let's assume that $A^\delta_n \subset A^{\bar \delta}_n$ and let $\exists f \in H^\delta(A^\delta_n, \U) \ : \ v \in \text{Im } f $, then $\exists x \in A^\delta_n \ \exists u \in \U \ \exists 0 \leq t \leq \delta\ : \ H^\delta(x, u)(t) = v$, which in turn obviously implies that $H^{\bar \delta}(x, u)(t) = v$. Thus we have 
\begin{equation}
v \in \text{Im } H^\delta(A^\delta_n, \U) \implies v \in \text{Im } H^{\bar \delta}(A^\delta_n, \U).
\end{equation}
which in turn implies
\begin{equation}
A^\delta_{n + 1} = \text{Im } H^{\delta}(A^\delta_n, \U) \subset \text{Im } H^{\bar \delta}(A^\delta_n, \U) \subset \text{Im } H^{\bar \delta}(A^{\bar \delta}_n, \U) = A^{\bar \delta}_{n + 1}
\end{equation}
Since $A^\delta_{0} \subset A^{\bar \delta}_{0}$ is satisfied, the above constitutes a proof by induction.
\end{proof}

\begin{lem}
Let Assumption \ref{asm:main} hold. For each $n \in \mathbb{N}$  and for each $\bar \delta$ there exists a compact set $X_n(\bar \delta) \subset \X$, such that for any sampled policy with sampling time no greater than $\bar \delta$
\begin{equation}
\forall t \in [0, n\delta] \ x(t) \in X_n(\bar \delta).
\end{equation}
\end{lem}
\begin{proof}
Let's assume that a sampled policy with sampling time $\delta$ is in place and $\delta \leq \bar \delta$. Evidently $x(t) \in \text{Im }H^\delta_{\lceil \frac{t}{\delta}\rceil}(\{x_0\}, \U)$, thus $t \in [0, n\delta]$ implies $x(t) \in \text{Im }H^\delta_{n}(\{x_0\}, \U)$. Thus $X_n(\bar \delta)$ can be chosen as the closure of $\text{Im }H^{\bar \delta}_{n}(\{x_0\}, \U)$.
\end{proof}

{\begin{rem}
\textit{Lemmas 1-4} are necessary to establish that the trajectories of the system are uniformly bounded. Once the uniform bound $X_{n}(\bar \delta)$ is obtained, one can utilize the extreme value theorem to construct bounds for the drift and the running objective \ie the bounds in \eqref{eqn:bounds}. Those enable us to derive \eqref{eqn:continuity-modulus} and \eqref{eqn:difference-bound} in the proof of \textit{Theorem 2}.
\end{rem}}

{\begin{thma}
Assumption \ref{asm:main} implies
\begin{equation}
	\label{thm:delta-thm}
\begin{aligned}
        & \min \limits_{\{ \rho_{i|k} \}_{i=1}^{N-1}} \left(\sum \limits_{i=1}^{N-1} Q^{\delta}(x_{i|k}| \rho_{i|k}) + J^*(x_{k+N-1})\right) - \\
        & \left(\sum_{i=1}^{N-1} \min \limits_{\rho_{i|k}} Q^{\delta}(x_{i|k}| \rho_{i|k}) + J^*(x_{k+N-1})\right) \xrightarrow{\delta \rightarrow 0} 0. \\
    \end{aligned}
\end{equation}
\end{thma}}
{
\begin{proof}
Let the sampling time $\delta$ be no greater than $\bar \delta$, let
\begin{equation}
\label{eqn:bounds}
\begin{aligned}
&\bar f := \max\limits_{x \in X_{N - 1}(\bar \delta), \ u \in \U} \nrm{f(x, u)}, \\
&\bar r := \max\limits_{x \in X_{N - 1}(\bar \delta), \ u \in \U} \abs{r(x, u)},
\end{aligned}
\end{equation}
and let $\omega_V(\cdot)$ denote a (non-decreasing) modulus of uniform continuity of $J^*$ over $X_{N - 1}(\bar \delta)$ (By Heine-Cantor theorem a continuous function is always uniformly continuous on a compact domain).

Now, observe that regardless of the choice of the sampled control input we have:
\begin{equation}
\label{eqn:continuity-modulus}
\begin{aligned}
& \abs{J^*(x_{k + i}) - J^*(x_k)} \leq \omega_V\left(\nrm{\int_{k\delta}^{(k + i)\delta} f(x(t), u(t)) \diff t }\right) \leq \\
& \omega_V\left(\int_{k\delta}^{(k + i)\delta} \nrm{f(x(t), u(t))} \diff t \right) \leq \omega_V(i\delta\bar f) \leq \omega_V(N\delta\bar f).
\end{aligned}
\end{equation}
Using \eqref{eqn:star} and \eqref{eqn:bar} we obtain:
\begin{equation}
\label{eqn:difference-bound}
\begin{aligned}
        & \min \limits_{\{ \rho_{i|k} \}_{i=1}^{N-1}} \left(\sum \limits_{i=1}^{N-1} Q^{\delta}(x_{i|k}| \rho_{i|k}) + J^*(x_{k+N-1})\right) - \\
        & \left(\sum_{i=1}^{N-1} \min \limits_{\rho_{i|k}} Q^{\delta}(x_{i|k}| \rho_{i|k}) + J^*(x_{k+N-1})\right) =  \\
& \int_{k\delta}^{(k + N - 1)\delta}e^{-\gamma t}r\big(\bar x(t), \bar \rho(t)\big)\diff t  -\\ 
& \int_{k\delta}^{(k + N - 1)\delta}e^{-\gamma t}r\big(x^*(t), \rho^*(t)\big)\diff t   + \\
& \sum \limits_{i=1}^{N - 1} J^*(\bar x_{k+i}) - \sum \limits_{i=1}^{N - 1} J^*(x^*_{k+i}) + \\
& J^*(\bar x_{k + N - 1}) - J^*(x^*_{k + N - 1}) \geq \\
& -2e^{-\gamma}\bar r (N - 1)\delta  \\
& -\sum \limits_{i=1}^{N - 1}\abs{J^*(\bar x_{k+i}) - J^*(x_k) + J^*(x_k) - J^*(x^*_{k + i})} \\
&-\abs{J^*(\bar x_{k+N - 1}) - J^*(x_k) + J^*(x_k) - J^*(x^*_{k + N - 1})} 
    \end{aligned}
\end{equation}
{Note, that \eqref{eqn:continuity-modulus} implies
\begin{multline}
-\abs{J^*(\bar x_{k+i}) - J^*(x_k) + J^*(x_k) - J^*(x^*_{k + i})} \geq \\
-\abs{J^*(\bar x_{k+i}) - J^*(x_k)} - \abs{J^*(x_k) - J^*(x^*_{k + i})} \geq -2\omega_V(N\delta\bar f).
\end{multline}
Thus we have 
\begin{multline}
 \min \limits_{\{ \rho_{i|k} \}_{i=1}^{N-1}} \left(\sum \limits_{i=1}^{N-1} Q^{\delta}(x_{i|k}| \rho_{i|k}) + J^*(x_{k+N-1})\right) - \\
         \left(\sum_{i=1}^{N-1} \min \limits_{\rho_{i|k}} Q^{\delta}(x_{i|k}| \rho_{i|k}) + J^*(x_{k+N-1})\right)  \\
\geq  -2e^{-\gamma}\bar r (N - 1)\delta  -2N\omega_V(N\delta\bar f) \xrightarrow{\delta \rightarrow 0} 0.
\end{multline}
Using the squeeze theorem together with \eqref{eqn:greater} we obtain the statement of \textit{Theorem 2} from the above.}
\end{proof}}
{
\begin{thma}
Let
\begin{multline}
 \{ \rho^{\ast}_{i|k} \}_{i=1}^{N-1} := \argmin \limits_{\{ \rho_{i|k} \}_{i=1}^{N-1}} Q^\circ(x_k|\{\rho_{i|k}\}_{i=1}^{N-1}),\\
\end{multline}
then  $\{ \rho^{\ast}_{i|k} \}_{i=1}^{N-1}$ are optimal.
\end{thma}
\begin{proof}
Note that the right hand side of \eqref{eqn:complemented} can be rearranged in the following way:
\begin{multline}
 Q^\circ(x_k|\{\rho_{i|k}\}_{i=1}^{N-1}) := \\
\sum \limits_{i=1}^{N-1} \left(\underbrace{\sum_{j=1}^{i - 1}r(x_{k + j - 1}, \rho_{j | k}) + Q^{\delta} (x_{i|k}, \rho_{i|k})}\limits_{\psi_{i}(x_k|\{\rho_{j|k}\}_{j=1}^{i}) :=}\right) + \\
 \underbrace{\sum_{j=1}^{N - 1}r(x_{k + j - 1}, \rho_{j | k}) + J^*(x_{k+N-1})}\limits_{\psi_{N}(x_k|\{\rho_{j|k}\}_{j=1}^{N-1})}.\\
\end{multline}
Let $\{\rho^{\ast}_{i|k} \}_{i=1}^{N-1}$ be optimal control inputs. By Bellman's principle of optimality we have:
\begin{multline}
\label{eqn:term-optimality}
\min\limits_{\{\rho_{j|k}\}_{j=1}^{i}}\psi_{i}(x_k|\{\rho_{j|k}\}_{j=1}^{i}) = \psi_{i}(x_k|\{\rho^{\ast}_{j|k}\}_{j=1}^{i}), \ i = 1\dots N-1,\\
\min\limits_{\{\rho_{j|k}\}_{j=1}^{N-1}}\psi_{N}(x_k|\{\rho_{j|k}\}_{j=1}^{N-1}) = \psi_{N}(x_k|\{\rho^{\ast}_{j|k}\}_{j=1}^{N-1}).
\end{multline}
Now, note that 
\begin{multline}
Q(x_k|\{\bar \rho_{i|k}\}_{i=1}^{N-1}) := \min_{\{ \rho_{i|k} \}_{i=1}^{N-1}} Q^\circ(x_k|\{\rho_{i|k}\}_{i=1}^{N-1}) = \\
 \min_{\{ \rho_{i|k} \}_{i=1}^{N-1}} \left(\sum \limits_{i=1}^{N-1} \psi_{i}(x_k|\{\rho_{j|k}\}_{j=1}^{i})  +  \psi_{N}(x_k|\{\rho_{j|k}\}_{j=1}^{N-1})\right)\\
 \geq \sum \limits_{i=1}^{N-1} \min\limits_{\{\rho_{j|k}\}_{j=1}^{i}}\psi_{i}(x_k|\{\rho_{j|k}\}_{j=1}^{i})  + \\
  \min\limits_{\{\rho_{j|k}\}_{j=1}^{N-1}}\psi_{N}(x_k|\{\rho_{j|k}\}_{j=1}^{N-1}) =  \\
    \sum \limits_{i=1}^{N-1} \psi_{i}(x_k|\{\rho^{\ast}_{j|k}\}_{j=1}^{i})  +  \psi_{N}(x_k|\{\rho^{\ast}_{j|k}\}_{j=1}^{N-1}) = \\
     Q(x_k|\{\rho^{\ast}_{i|k}\}_{i=1}^{N-1}).\\
\end{multline}
\end{proof}
But at the same time, since $\{\bar \rho_{i|k}\}_{i=1}^{N-1}$ is the minimizer, we have
\begin{multline}
Q(x_k|\{\bar \rho_{i|k}\}_{i=1}^{N-1}) \leq Q(x_k|\{\rho^{\ast}_{i|k}\}_{i=1}^{N-1}). \\
\end{multline}
Thus by double inclusion
\begin{multline}
Q(x_k|\{\bar \rho_{i|k}\}_{i=1}^{N-1}) = Q(x_k|\{\rho^{\ast}_{i|k}\}_{i=1}^{N-1}). \\
\end{multline}
The latter implies
\begin{multline}
\sum \limits_{i=1}^{N-1} \min\limits_{\{\rho_{j|k}\}_{j=1}^{i}}\psi_{i}(x_k|\{\rho_{j|k}\}_{j=1}^{i})  + \\
  \min\limits_{\{\rho_{j|k}\}_{j=1}^{N-1}}\psi_{N}(x_k|\{\rho_{j|k}\}_{j=1}^{N-1}) = \\ \sum \limits_{i=1}^{N-1} \psi_{i}(x_k|\{\bar \rho_{j|k}\}_{j=1}^{i})  +  \psi_{N}(x_k|\{\bar \rho_{j|k}\}_{j=1}^{N-1}).
\end{multline}
Together with \eqref{eqn:term-optimality} this yields
\begin{multline}
\psi_{i}(x_k|\{\bar \rho_{j|k}\}_{j=1}^{i}) = \psi_{i}(x_k|\{\rho^{\ast}_{j|k}\}_{j=1}^{i}), \ i = 1\dots N-1,\\
\psi_{N}(x_k|\{\bar \rho_{j|k}\}_{j=1}^{N-1}) = \psi_{N}(x_k|\{\rho^{\ast}_{j|k}\}_{j=1}^{N-1}),
\end{multline}
where the last of these identities simplifies to
\begin{multline}
\sum_{j=1}^{N - 1}r(\bar x_{k + j - 1}, \bar \rho_{j | k}) + J^*(\bar x_{k+N-1}) = \\
\sum_{j=1}^{N - 1}r(x^\ast_{k + j - 1}, \rho^\ast_{j | k}) + J^*(x^\ast_{k+N-1}), \\
\end{multline}
where $\bar x$, $x^{\ast}$ imply trajectories resulting from respective policies $\bar \rho$, $\rho^{\ast}$. The above identity is equivalent to the statement of the theorem.}

\bibliographystyle{supp/IEEEtran}
\bibliography{bib/RL-safe, bib/AIDA, bib/RL-stab, bib/MPC-determ, bib/MPC-stoch, bib/MPC-ML, bib/MPC-stab, bib/PO-related, bib/Osinenko, bib/RL-ML, bib/MPC-industr, bib/RL-industr, bib/Q-collapse, bib/nonlin-ctrl}

\end{document}